%% file: main.tex
\documentclass[11pt]{article}

\usepackage{fullpage}
\usepackage{graphicx} 
\usepackage{amsmath}
\usepackage{pgfplots}

\usepackage{hyperref}
\usepackage{url}
\usepackage{subcaption}

\input{macros}

\newcommand{\Gauss}{\mathrm{Gauss}}
\newcommand{\Cauchy}{\mathrm{Cauchy}}

\title{On the effect of the activation function \\
       on the distribution of hidden nodes \\
       in a deep network \\}
\author{Philip M. Long\thanks{Authors ordered alphabetically.} 
        and 
        Hanie Sedghi\footnotemark[1] \\
        Google Brain \\}
\date{}

\begin{document}
\maketitle
\begin{abstract}
We analyze the joint probability distribution on the lengths of the
vectors of hidden variables in different layers of a fully connected
deep network, when the weights and biases are chosen randomly according to
Gaussian distributions, and the input is in $\{ -1, 1\}^N$.  We show
that, if the activation function $\phi$ satisfies a minimal set of
assumptions, satisfied by all activation functions that we know that
are used in practice, then, as the width of the network gets large,
the ``length process'' converges in probability to a length map
that is determined as a simple function of the variances of the
random weights and biases, and the activation function $\phi$.

We also show that this convergence may fail for $\phi$ that violate our assumptions.
\end{abstract}

\section{Introduction}

The size of the weights of a deep network must be managed delicately.  If they are too large,
signals blow up as they travel through the network, leading to numerical problems, and if they
are too small, the signals fade away.  The practical state of the art in deep learning made a
significant step forward due to schemes for initializing the weights that aimed in different ways at 
maintaining roughly the same scale for the hidden variables 
before and after a layer
\cite{lecun2012efficient,glorot2010understanding}.  
Later work \cite{he2015delving,poole2016exponential,daniely2016toward}
took into account the effect of the non-linearities on the length dynamics of a deep network,
informing initialization policies in a more refined way.

In this paper, we continue this line of work, theoretically analyzing
what might be called the ``length process''.  That is, for a given
input, chosen for simplicity from $\{ -1, 1 \}^N$, we study the
probability distribution over the lengths of the vectors of hidden
variables, when the parameters of a deep network are chosen randomly.
We analyze the case of fully connected networks, with the same
activation function $\phi$ at each hidden node and $N$ hidden
variables in each layer.  As in \cite{poole2016exponential}, we
consider the case where weights between nodes are chosen from a
zero-mean Gaussian with variance $\sigma_w^2/N$, and where the biases
are chosen from a zero-mean distribution with variance $\sigma_b^2$.

Our first result holds for activation functions $\phi$ that satisfy
the following properties: (a) the restriction of $\phi$ to any finite
interval is bounded; (b) as $z$ gets large,\footnote{Here
$o(z^2)$ denotes any function of $z$ that grows strictly more
slowly than $z^2$, such as $z^{2-\epsilon}$ for $\epsilon > 0$.}
$|\phi(z)| \leq \exp(o(z^2))$,
(c) $\phi$ is measurable.  We refer to
such $\phi$ as {\em permissible}.  Note that conditions (a) and (c)
both hold for any non-decreasing $\phi$.

We show that, for all permissible $\phi$ and all $\sigma_w$ and
$\sigma_b$, as $N$ gets large, the length process converges in
probability to a length map that is a simple function of $\phi$, $\sigma_w$ and $\sigma_b$.  This length map was first discovered
in
\cite{poole2016exponential}, 
where it was claimed that it holds for all $\phi$; 
it has since been used
in a number of other papers \cite{schoenholz2016deep, yang2017mean, pennington2017resurrecting, lee2018deep, xiao2018dynamical, chen2018dynamical, pennington2018emergence, hayou2018selection}.

In Section~\ref{s:diversity}, 
to motivate our new analysis, we provide examples of $\phi$ that are
not permissible that lead to length processes with arguably
surprising properties.  For example, we show that, for arbitrarily
small positive $\sigma_w$, even if $\sigma_b = 0$, for $\phi(z) =
1/z$, the distribution of values of each of the hidden nodes in the
second layer diverges as $N$ gets large.  For finite $N$, each node
has a Cauchy distribution, which already has infinite variance, and as
$N$ gets large, the scale parameter of the Cauchy distribution gets
larger, leading to divergence.  We also show that the hidden variables
in the second layer may not be independent, even for some permissible
$\phi$ like the ReLU.  The results of this section contradict claims
made in \cite{poole2016exponential}.

Section~\ref{s:experiments} describes some simulation experiments
verifying some of the findings of the paper, and illustrating the
dependence among the values of the hidden nodes.

Our analysis of the convergence of the length map borrows ideas from
Daniely, et al.\ 
\cite{daniely2016toward}, who studied the properties of the
mapping from inputs to hidden representations resulting from random
Gaussian initialization.  
Their theory applies in the case of activation functions with certain
smoothness properties, and to a wide variety of architectures.  
Our analysis treats a wider variety of values of $\sigma_w$ and
$\sigma_b$, and uses weaker assumptions on $\phi$.


\section{Preliminaries}

\subsection{Notation}
For $n \in \mathbb{N}$, we use $[n]$ to denote
the set $\lbrace 1, 2, \dotsc, n \rbrace$.  If $T$ is a
$n \times m \times p$ tensor, then, for $i \in [n]$, let
$T_{i,:,:} = \langle T_{i,j,k} \rangle_{jk}$, and define
$T_{i,j,:}$, etc., analogously.

\subsection{The finite case}
%
Consider a deep fully connected width-$N$ 
network with $D$ layers.  Let $W \in \R^{D \times N \times N}$.
An activation function $\phi$ maps $\R$ to $\R$;  we will also use
$\phi$ to denote the function from $\R^N$ to $\R^N$ obtained by applying $\phi$ componentwise.
Computation of the 
neural activity vectors $x_{0,:},...,x_{D,:} \in \R^N$
and preactivations $h_{1,:},...,h_{D,:} \in \R^N$ 
proceeds in the standard way as follows:
\begin{align*}
h_{\ell,:} = W_{\ell,:,:} x_{\ell-1,:}+ b_{\ell,:}
 \quad
 x_{\ell,:} = \phi(h_{\ell,:}), \quad  \quad \text{for}~~\ell = 1, \dotsc, D.
\end{align*}

We will study the process arising from fixing an arbitrary input
$x_{0,:} \in \{ -1, 1 \}^N$ and choosing the parameters independently
at random: the entries of $W$ are sampled from $\Gauss\left(0,
\frac{\sigma^2_w}{N}\right)$, and the entries of $b$ from
$\Gauss\left(0, \sigma^2_b\right)$.
For each $\ell \in [D]$, define
$q_{\ell} = \frac{1}{N} \sum_{i = 1}^N h_{\ell,i}^2$.

Note that for all $\ell \geq 1$, all the components of $h_{\ell,:}$ 
and $x_{\ell,:}$ are identically
distributed.

\subsection{The wide-network limit}

For the purpose of defining a limit, assume that,
for a fixed, arbitrary function $\chi : \N \rightarrow \{ -1, 1\}$,
for finite $N$, we have $x_{0,:} = (\chi(1),...,\chi(N))$.
For $\ell > 0$, if the limit exists 
(in the sense of ``convergence in distribution''),
let $\ux_\ell$ be a random variable whose distribution is the limit of
the distribution of $x_{\ell,1}$ as $N$ goes to infinity.  
Define $\uh_\ell$ and
$\uq_\ell$ similarly.

\subsection{Total variation distance}

If $P$ and $Q$ are probability distributions, then
$d_{TV}(P,Q) = \sup_E P(E) - Q(E)$, and if
$p$ and $q$ are their densities,
$d_{TV}(P,Q) = \frac{1}{2} \int | p(x) - q(x) | \; dx.$

\section{Convergence in probability}
\label{s:positive}

In this section we characterize the length map of the
hidden nodes of a deep network, for all activation functions
satisfying the following assumptions.
\begin{definition}
An activation function $\phi$ is {\em permissible} if,
(a) the restriction of $\phi$ to any finite interval is bounded;
(b) $|\phi(x)| \leq \exp(o(x^2))$ as $|x|$ gets large.\footnote{
This condition may be expanded as follows, 
$\mathrm{limsup}_{x \rightarrow \infty} \frac{\log | \phi(x) |}{x^2} = 0$
and 
$\mathrm{limsup}_{x \rightarrow -\infty} \frac{\log | \phi(x) |}{x^2} = 0$.};
and (c) $\phi$ is measurable.
\end{definition}

Conditions (b) and (c) ensure that a key
integral can be computed.  The proof 
of Lemma~\ref{l:integrable} is in Appendix~\ref{a:integrable}.
\begin{lemma}
\label{l:integrable}
If $\phi$ is permissible, then, for all positive constants $c$,
the function $g$ defined by $g(x) = \phi(c x)^2 \exp(-x^2/2)$ is
integrable.
\end{lemma}

Now, we recall the definition of
a length map from \cite{poole2016exponential};
we will prove that the the length process converges to
this length map.
Define $\tq_0,...,\tq_D$ 
and $\tr_0,...,\tr_D$ 
recursively as follows.  First $\tq_0 = \tr_0 = 1$.
Then, for $\ell > 0$,
\[
\tq_{\ell}
  = \sigma_w^2 \tr_{\ell-1} + \sigma_b^2
\]
and
\[
\tr_{\ell} = \Ebb_{z \in \Gauss(0,1)} [ \phi (\sqrt{ \tq_{\ell}} z)^2 ].
\]
If $\phi$ is permissible, then, since $\phi (c z)^2 \exp(-z^2/2)$ is
integrable for all $c$, we have that
$\tq_0,...,\tq_D,\tr_0,...,\tr_D$ are well-defined
finite real numbers.  

The following theorem shows that the
length map $q_0,...,q_D$ converges in probability
to $\tq_0,...,\tq_D$.
\begin{theorem}
\label{t:conv.prob}
For any permissible $\phi$, 
$\sigma_w, \sigma_b \geq 0$, any depth $D$, and any $\epsilon,
\delta > 0$, there is an $N_0$ such that, for all $N \geq N_0$, 
with probability $1 - \delta$, for
all $\ell \in \{ 0,...,D \}$, we have
$| q_{\ell} - \tq_{\ell} | \leq \epsilon.$
\end{theorem}

The rest of this section is devoted to proving Theorem~\ref{t:conv.prob}.
Our proof will use the weak law of large numbers.
\begin{lemma}[\cite{feller2008introduction}]
\label{l:weak.law}
For any random variable $X$ with a finite expectation, and any
$\epsilon,\delta > 0$, there is an $N_0$ such that, for all $N \geq
N_0$, if $X_1,...,X_N$ are i.i.d.\ with the same distribution as $X$,
then
\[
\Pr\left( \left| \Ebb[X] - \frac{1}{N} \sum_{i=1}^N X_i \right| > \epsilon\right) \leq \delta.
\]
\end{lemma}
In order to divide our analysis into cases, we need
the following lemma, whose proof is in Appendix~\ref{a:non_zero}.
\begin{lemma}
\label{l:non_zero}
If $\phi$
is permissible and not zero a.e.,
for all
$\sigma_w > 0$, for all $\ell \in \{ 0,...,D \}$,
$\tq_{\ell} > 0$
and $\tr_{\ell} > 0$.
\end{lemma}
We will also need a lemma that shows that
small changes in $\sigma$ lead to small changes in
$\Gauss(0,\sigma^2)$.
\begin{lemma}[see \cite{klartag2007central}]
\label{l:gauss}
There is an absolute constant $C$ such that,
for all $\sigma_1, \sigma_2 > 0$, \\
$d_{TV}(\Gauss(0,\sigma_1^2), \Gauss(0,\sigma_2^2))
 \leq C \frac{| \sigma_1 - \sigma_2 | }{\sigma_1}$.
\end{lemma}
The following technical lemma, which shows that tail bounds hold
uniformly over different choices of $q$, is proved in
Appendix~\ref{a:tail}.
\begin{lemma}
\label{l:tail}
If $\phi$ is permissible, 
for all $0 < r \leq s$, for all $\beta > 0$, there is an
$a \geq 0$ such that, for all $q \in [r,s]$,
$
\int_a^{\infty} \phi(\sqrt{q} z)^2 \exp(-z^2/2) \;dz \leq \beta
$
and 
$
\int_{-\infty}^{-a} \phi(\sqrt{q} z)^2 \exp(-z^2/2) \;dz \leq \beta.
$
\end{lemma}

Armed with these lemmas, we are ready to
prove Theorem~\ref{t:conv.prob}.  

First, if $\phi$ is zero a.e., 
or if $\sigma_w = 0$, 
Theorem~\ref{t:conv.prob}
follows directly from Lemma~\ref{l:weak.law}, 
together
with a union bound over the layers.  Assume for the
rest of the proof that $\phi(x)$ is 
non-zero on a set of positive measure,
and that $\sigma_w > 0$, so that
$\tq_{\ell} > 0$
and $\tr_{\ell} > 0$ for all $\ell$.  

For each $\ell \in [D]$, define
$
r_{\ell} = \frac{1}{N} \sum_{i=1}^N x_{\ell,i}^2.
$

Our proof of Theorem~\ref{t:conv.prob} is by induction.  
The inductive hypothesis is 
that, for any $\epsilon, \delta > 0$ 
there is an $N_0$ such that, if $N \geq N_0$, then,
with probability $1 - \delta$, for all $\ell' \leq \ell$,
$| q_{\ell'} - \tq_{\ell'} | \leq \epsilon$
and $| r_{\ell'} - \tr_{\ell'} | \leq \epsilon$.

The base case holds because $q_0 = \tq_0 = r_0 = \tr_0 = 1$, no matter
what the value of $N$ is.

Now for the induction step; choose $\ell > 0$,
$0 < \epsilon < \min \{ \tq_{\ell}/4, \tr_{\ell} \}$ and $0 < \delta \leq 1/2$.  
(Note that these choices are without loss of generality.)
Let $\epsilon' \in (0, \epsilon)$ take a value that will be
described later, using quantities from the analysis.
By the inductive hypothesis, whatever the value
of $\epsilon'$,
there is an $N_0'$ such that, if $N \geq N_0'$, then,
with probability $1 - \delta/2$, 
for all $\ell' \leq \ell - 1$,
we have
$| q_{\ell'} - \tq_{\ell'} | \leq \epsilon'$
and $| r_{\ell'} - \tr_{\ell'} | \leq \epsilon'$.
Thus, to
establish the inductive step, it suffices to show
that, after conditioning on the random choices before
the $\ell$th layer, if 
$| q_{\ell-1} - \tq_{\ell-1} | \leq \epsilon'$,
and $| r_{\ell-1} - \tr_{\ell-1} | \leq \epsilon'$,
there is an $N_{\ell}$ such that, if $N \geq N_{\ell}$,
then with
probability at least $1 - \delta/2$ with respect only
to the random choices of $W_{\ell,:,:}$ and $b_{\ell,:}$, 
that $| q_{\ell} - \tq_{\ell} | \leq \epsilon$
and $| r_{\ell} - \tr_{\ell} | \leq \epsilon$.  Given such an $N_{\ell}$, the inductive step
can be satisfied by letting $N_0$ be the maximum of
$N_0'$ and $N_{\ell}$.

Let us do that.
To simplify the notation, for the rest of the proof of the inductive step,
let us condition on outcomes of the layers before
layer $\ell$; all expectations and probabilities will 
concern the randomness only in the $\ell$th layer.  
Let us further assume that
$| q_{\ell-1} - \tq_{\ell-1} | \leq \epsilon'$
and $| r_{\ell-1} - \tr_{\ell-1} | \leq \epsilon'$.

Recall
that 
$q_{\ell} = \frac{1}{N} \sum_{i=1}^N h_{\ell,i}^2$.
Since 
the values of
$h_{\ell-1,1}, ... ,h_{\ell-1,N}$ have been fixed by conditioning, each
component of $h_{\ell,i}$ is obtained by
taking the dot-product of $x_{\ell-1,:} = \phi(h_{\ell-1,:})$
with $W_{\ell,i,:}$ and adding
an independent $b_{\ell,i}$.  Thus, conditioned on
$h_{\ell-1,1}, ... ,h_{\ell-1,N},$ we have that
$h_{\ell,1}, ... ,h_{\ell,N}$ are independent.
Also, since $x_{\ell-1,:}$ is fixed by conditioning, each $h_{\ell,i}$
has an identical Gaussian distribution.  

Since each component of $W$ and $b$ has zero
mean, each $h_{\ell,i}$ has zero mean.  

Choose an arbitrary $i \in [N]$.
Since $x_{\ell-1,:}$ is fixed by conditioning and
$W_{\ell,i,1},...,W_{\ell,i,N}$ and $b_{\ell,i}$ are independent,
\begin{equation}
\label{e:by.x}
\Ebb[q_{\ell} ]
  = \Ebb[h_{\ell,i}^2]
  = \sigma_b^2 + \frac{\sigma_w^2}{N} \sum_j x_{\ell-1,j}^2
  = \sigma_b^2 + \sigma_w^2 r_{\ell-1}
  \eqdef \oq_{\ell}.
\end{equation}
We wish to emphasize the $\oq_{\ell}$ is determined as a function
of random outcomes before the $\ell$th layer, and thus
a fixed, nonrandom quantity, regarding the randomization
of the $\ell$th layer.  By the inductive hypothesis, we have
\begin{equation}
\label{e:q.close.suffices}
| \Ebb[q_{\ell}] - \tq_{\ell} |
 = | \Ebb[h_{\ell,i}^2] - \tq_{\ell} |
 = | \oq_{\ell} - \tq_{\ell} |
 = \sigma_w^2 | r_{\ell-1} - \tr_{\ell-1} |
  \leq \epsilon' \sigma_w^2.
\end{equation}
The key consequence of this might be paraphrased by saying
that, to establish the portion of the inductive step regarding
$q_{\ell}$, it suffices for $q_{\ell}$ to be close to
its mean.  Now, we want to prove something similar for
$r_{\ell}$.  
We have
\begin{align*}
\Ebb[r_{\ell}] 
 = \frac{1}{N} \sum_{i=1}^N \Ebb[x_{\ell,i}^2] 
 = \frac{1}{N} \sum_{i=1}^N \Ebb[\phi(h_{\ell,i})^2] 
 = \Ebb[\phi(h_{\ell,1})^2],
\end{align*}
since, recalling that we have conditioned on previous layers,
$h_{\ell,1},...,h_{\ell,N}$ are i.i.d.  
Since $h_{\ell,i} \sim \Gauss(0, \oq_{\ell})$, we have
\begin{align*}
\Ebb[r_{\ell}] 
 & = \Ebb_{z \sim \Gauss(0,\oq_{\ell})}  [\phi(z)^2] 
 = \Ebb_{z \sim \Gauss(0,1)}  [\phi(\sqrt{\oq_{\ell}} z)^2] 
= \sqrt{\frac{1}{2 \pi}} \int \phi(\sqrt{\oq_{\ell}} z)^2 \exp(-z^2/2) \;dz. \\
\end{align*}
which gives
\begin{align*}
| \Ebb[r_{\ell}] - \tr_{\ell} |
  & \leq \left| \Ebb_{z \sim \Gauss(0,\oq_{\ell})} [\phi(z)^2]
           - \Ebb_{z \sim \Gauss(0,\tq_{\ell})} [\phi(z)^2]
       \right|.
\end{align*}
Since $| \oq_{\ell} - \tq_{\ell} | \leq \epsilon' \sigma_w^2$
and we may choose $\epsilon'$ to ensure
$\epsilon' \leq \frac{\tq_{\ell}}{2 \sigma_w^2}$,
we have
$
\tq_{\ell}/2 \leq \oq_{\ell} \leq 2 \tq_{\ell}.
$

For $\beta > 0$ and $\kappa \in (0,1/2)$ to be named later,
by Lemma~\ref{l:tail}, we can choose $a$ such that,
for all $q \in [\tq_{\ell}/2,2 \tq_{\ell}]$,
\begin{align*}
& \int_{-\infty}^{-a} \phi(\sqrt{q} z)^2 \exp(-z^2/2) \;dz
   \leq \beta/2 
\;\;\;\mbox{ and }\;\;\;
\int_{a}^{\infty} \phi(\sqrt{q} z)^2 \exp(-z^2/2) \;dz
   \leq \beta/2 \\
\end{align*}
and
$
\frac{1}{\sqrt{2 \pi q}}
 \int_{-a}^a \exp\left(-\frac{z^2}{2 q}\right) \;dz
 \geq 1 - \kappa.
$
Choose such an $a$.

We claim that
$
\left|  \int_{-a}^a \phi(\sqrt{q} z)^2 \exp(-z^2/2) \;dz
   -    \int \phi(\sqrt{q} z)^2 \exp(-z^2/2) \;dz \right|
    \leq \beta
$
for all $\tq_{\ell}/2 < q \leq 2 \tq_{\ell}$.  
Choose such a $q$.  We have
\begin{align*}
& \left|  \int_{-a}^a \phi(\sqrt{q} z)^2 \exp(-z^2/2) \;dz
   -    \int \phi(\sqrt{q} z)^2 \exp(-z^2/2) \;dz \right| \\
& =  \int_{-\infty}^{-a} \phi(\sqrt{q} z)^2 \exp(-z^2/2) \;dz
   +    \int_a^{\infty} \phi(\sqrt{q} z)^2 \exp(-z^2/2) \;dz  \\
& \leq  
   2 \max \left\{ \int_{-\infty}^{-a} \phi(\sqrt{q} z)^2 \exp(-z^2/2) \;dz,
        \int_a^{\infty} \phi(\sqrt{q} z)^2 \exp(-z^2/2) \;dz \right\} \\
& \leq  
   \beta. \\
\end{align*}
So now we are trying to bound
$
\left|  \int_{-a}^a \phi(\sqrt{\oq_{\ell}} z)^2 \exp(-z^2/2) \;dz
   -   \int_{-a}^a \phi(\sqrt{\tq_{\ell}} z)^2 \exp(-z^2/2) \;dz \right|
$
using $\tq_{\ell}/2 \leq \oq_{\ell} \leq 2 \tq_{\ell}$.

Using changes of variables, we have
\begin{align*}
& \left|  \int_{-a}^a \phi(\sqrt{\oq_{\ell}} z)^2 \exp(-z^2/2) \;dz
   -    \int_{-a}^a \phi(\sqrt{\tq_{\ell}} z)^2 \exp(-z^2/2) \;dz \right| \\
& = \left|  \frac{1}{\sqrt{\oq_{\ell}}} \int_{-a \sqrt{\oq_{\ell}}}^{a \sqrt{\oq_{\ell}}}  \phi(z)^2 
                 \exp\left(-\frac{z^2}{2\oq_{\ell}}\right) \;dz
   -   
     \frac{1}{\sqrt{\tq_{\ell}}}
       \int_{-a \sqrt{\tq_{\ell}}}^{a \sqrt{\tq_{\ell}}} 
        \phi(z)^2 
   \exp\left(-\frac{z^2}{2\tq_{\ell}}\right)
            \;dz \right|.
\end{align*}
Since $\phi$ is permissible,
$\phi^2$ is bounded on 
$[- a \sqrt{2 \tq_{\ell}}, a  \sqrt{2 \tq_{\ell}}]$.
If $P$ is the distribution obtained by conditioning
$\Gauss(0,\oq_{\ell})$ on $[-a \sqrt{\oq_{\ell}}, a \sqrt{\oq_{\ell}} ]$,
and $\tilde{P}$ by conditioning
$\Gauss(0,\tq_{\ell})$ on $[-a \sqrt{\tq_{\ell}}, a \sqrt{\tq_{\ell}} ]$, then
if 
$M = \sqrt{2 \pi} 
     \sup_{z \in [- a \sqrt{2 \tq_{\ell}}, a  \sqrt{2 \tq_{\ell}}]} \phi(z)^2$,
since $\oq_{\ell} \leq 2 \tq_{\ell}$,
\[
\left|  \frac{1}{\sqrt{\oq_{\ell}}} \int_{-a \sqrt{\oq_{\ell}}}^{a \sqrt{\oq_{\ell}}}  \phi(z)^2 
                 \exp(-\frac{z^2}{2\oq_{\ell}}) \;dz
   -   
     \frac{1}{\sqrt{\tq_{\ell}}}
       \int_{-a \sqrt{\tq_{\ell}}}^{a \sqrt{\tq_{\ell}}} 
        \phi(z)^2 
   \exp(-\frac{z^2}{2\tq_{\ell}})
            \;dz \right|
 \leq M d_{TV}(P, \tilde{P}).
\]
But since, for $\kappa < 1/2$, conditioning on an event 
of probability at least 
$1 - \kappa$ only changes a distribution by
total variation distance at most $2 \kappa$,
and therefore, applying Lemma~\ref{l:gauss} along with the
fact that $| \oq_{\ell} - \tq_{\ell} | \leq \epsilon' \sigma_w^2$,
for the constant $C$ from Lemma~\ref{l:gauss}, we get
\begin{align*}
d_{TV}(P, \tilde{P})
 & \leq 4 \kappa + d_{TV}(\Gauss(0,\oq_{\ell}),\Gauss(0,\tq_{\ell})) \\
 & \leq 4 \kappa 
    + \frac{C | \sqrt{\oq_{\ell}} - \sqrt{\tq_{\ell}}|}{\sqrt{\tq_{\ell}}} \\
 & = 4 \kappa 
    + \frac{C | \oq_{\ell} - \tq_{\ell}|}
           {| \sqrt{\oq_{\ell}} + \sqrt{\tq_{\ell}}| \sqrt{\tq_{\ell}}} \\
 & \leq 4 \kappa 
    + \frac{C \epsilon' \sigma_w^2}{\tq_{\ell}}. \\
\end{align*}
Tracing back, 
we have
\[
\left|  \int_{-a}^a \phi(\sqrt{\oq_{\ell}} z)^2 \exp(-z^2/2) \;dz
   -    \int_{-a}^a \phi(\sqrt{\tq_{\ell}} z)^2 \exp(-z^2/2) \;dz \right|
\leq 
 M \left(4 \kappa + \frac{C \epsilon' \sigma_w^2}{\tq_{\ell}}\right)
\]
which implies
\begin{align*}
| \Ebb[r_{\ell}] - \tr_{\ell} |
& \leq
\left|  \int \phi(\sqrt{\oq_{\ell}} z)^2 \exp(-z^2/2) \;dz
   -    \int \phi(\sqrt{\tq_{\ell}} z)^2 \exp(-z^2/2) \;dz \right| \\
& \leq 
 M \left(4 \kappa + \frac{C \epsilon' \sigma_w^2}{\tq_{\ell}}\right)
  + 2 \beta.
\end{align*}
If $\kappa = \min\{ \frac{\epsilon}{24 M}, \frac{1}{3} \}$,
$\beta = \frac{\epsilon}{12}$, and
$
\epsilon' = \min \left\{ \frac{\epsilon}{2}, 
                     \frac{\epsilon}{2 \sigma_w^2}, 
                     \frac{\tq_{\ell}}{2 \sigma_w^2}, 
                     \frac{\tq_{\ell} \epsilon}{6 C M \sigma_w^2}
                   \right\}
$
this implies
$
| \Ebb[r_{\ell}] - \tr_{\ell} | \leq \epsilon/2.
$

Recall that $q_\ell$ is an average of $N$ 
identically distributed
random variables with a mean between $0$ and $2 \tq_{\ell}$ 
(which is therefore finite) and
$r_\ell$ is an average of $N$ identically distributed random variables, 
each with mean 
between $0$ and $\tr_{\ell} + \epsilon/2 \leq 2 \tr_{\ell}$.
Applying the weak law of large numbers
(Lemma~\ref{l:weak.law}), there is an
$N_{\ell}$ such that, if $N \geq N_{\ell}$, with probability
at least $1 - \delta/2$, both
$| q_{\ell} - \Ebb[q_{\ell}]| \leq \epsilon/2$
and 
$| r_{\ell} - \Ebb[r_{\ell}]| \leq \epsilon/2$ hold,
which in turn implies
$| q_{\ell} - \tq_{\ell}| \leq \epsilon$
and 
$| r_{\ell} - \tr_{\ell}| \leq \epsilon$, completing the
proof of the inductive step, and therefore the proof of 
Theorem~\ref{t:conv.prob}.

\section{Diversity of behavior in the distribution of hidden nodes}
\label{s:diversity}

In this section, we show that, for some activation functions,
the probability distribution of hidden nodes can
have some surprising properties.

\subsection{Non-Gaussian}
\label{s:cauchy}

In this subsection, we will show that
the hidden variables are sometimes not Gaussian.
Our proof will refer to the Cauchy distribution.
\begin{definition}
\label{d:cauchy}
A distribution over the reals that, for $x_0 \in \R$ and $\gamma > 0$, has a density $f$ given by
$f(x) = \frac{1}{\pi \gamma \left[ 1 + \left( \frac{x - x_0}{\gamma} \right)^2 \right]}$
is a {\em Cauchy distribution},
denoted by $\Cauchy(x_0, \gamma)$.  
$\Cauchy(0, 1)$ is the {\em standard Cauchy distribution}.
\end{definition}

\begin{lemma}[\cite{hazewinkel2013encyclopaedia.cauchy}]
\label{l:stable}
If $X_1,...,X_n$ are i.i.d.\ random variables with a Cauchy distribution, then
$\frac{1}{n} \sum_{i=1}^n X_i$ has the same distribution.
\end{lemma}

\begin{lemma}[\cite{lupton1993statistics}]
\label{l:ratio}
If $U$ and $V$ are zero-mean normally distributed random variables with the
same variance, then $U/V$ has the standard Cauchy distribution.  
\end{lemma}

The following shows that 
there is a $\phi$ such that the limiting
$\uh_2$ is not defined.  It contradicts a claim made on
line 7 of Section A.1 of \cite{poole2016exponential}.
\begin{proposition}
\label{p:not.gaussian}
There is a $\phi$ such that, for every $\sigma_w > 0$, if $\sigma_b = 0$, then
(a) for finite $N$, $h_{2,1}$ does not have a Gaussian distribution, and
(b) $h_{2,1}$ diverges as $N$ goes to infinity.
\end{proposition}
\begin{proof}
Consider $\phi$ defined by 
$
\phi(y) = \left\{
            \begin{array}{ll}
              1/y & \mbox{if $y \neq 0$} \\
              0 & \mbox{if $y = 0$}.
            \end{array}
              \right.
$

Fix a value of $N$ and $\sigma_w > 0$, and take $\sigma_b = 0$.
Each component of $h_{1,:}$ 
is a sum of zero-mean Gaussians with variance $\sigma_w^2/N$; 
thus, for all $i$, $h_{1,i} \sim \Gauss(0, \sigma_w^2)$.  Now, almost surely,
$
h_{2,1} = \sum_{j=1}^N W_{2,1,j} \phi(h_{1,j})
     = \sum_{j=1}^N W_{2,1,j} /h_{1,j}.
$
By Lemma~\ref{l:ratio}, for each $j$, $W_{2,1,j} /h_{1,j}$ 
has a Cauchy distribution,
and since $(N W_{2,1,1}),...,(NW_{2,1,N}) \sim \Gauss(0, N \sigma_w^2)$, 
recalling that $h_{1,1},...,h_{1,N} \sim \Gauss(0, \sigma_w^2)$,
we have that $N W_{2,1,1} /h_{1,1},...,N W_{2,1,N} /h_N^1$ 
are i.i.d. $\Cauchy(0,\sqrt{N})$.  Applying
Lemma~\ref{l:stable}, 
$
h_{2,1} = \sum_{j=1}^N W_{2,1,j} \phi(h_{2,j}) = \frac{1}{N} \sum_{j=1}^N N W_{2,1,j} \phi(h_{1,j})
$
is also $\Cauchy(0,\sqrt{N})$.

So, for all $N$, $h_{2,1}$ is $\Cauchy(0,\sqrt{N})$.
Suppose that $h_{2,1}$ converged in distribution to
some distribution $P$.  Since the cdf of $P$ can have at most
countably many discontinuities, we can cover the real line by
a countable set of finite-length intervals 
$[a_1,b_1], [a_2,b_2], ...$ whose endpoints
are points of continuity for $P$.  
Since $\Cauchy(0,\sqrt{N})$ converges to $P$ in
distribution, for any $i$, 
$
P([a_i,b_i]) 
 \leq \lim_{N \rightarrow \infty} \frac{| b_i - a_i |}{\pi \sqrt{N}} = 0.
$
Thus, the probability assigned by $P$ to the entire real line is
$0$, a contradiction.
\end{proof}

\subsection{Independence}

The following contradicts a claim made on line 
8 of Section A.1 of \cite{poole2016exponential}.
\begin{theorem} \label{thm:independence}
If $\phi$ is either the ReLU or the Heaviside
function, then, 
for every  $\sigma_w > 0$,  $\sigma_b \geq 0$, 
and $N \geq 2$, $(h_{2,1},...,h_{2,N})$ 
are not independent.
\end{theorem}

\begin{proof}
We will show that 
$\Ebb[ h_{2,1}^2 h_{2,2}^2 ] \neq \Ebb[ h_{2,1}^2] \Ebb[h_{2,2}^2]$,
which will imply that $h_{2,1}$ and $h_{2,2}$ are not independent.

As mentioned earlier,
because each component of $h_{1,:}$ is the dot product of $x_{0,:}$
with an independent row of $W_{1,:,:}$ plus an independent component
of $b_{1,:}$, the components of $h_{1,:}$ are independent, and
since $x_{1,:} = \phi(h_{1,:})$, this implies that the components of
$x_{1,:}$ are independent.  Since each row of
$W_{1,:,:}$ and each component of the bias vector has the same
distribution, $x_{1,:}$ is i.i.d.

We have
\begin{align*}
\Ebb[h_{2,1}^2]&=\Ebb\left[ \left[ \left( \sum_{i \in [N]} W_{2,1,i} x_{1,i} \right) +b_{2,1} \right]^2\right]  
  \\
          &
 =  \sum_{(i,j) \in [N]^2} \Ebb\left[ W_{2,1,i} W_{2,1,j} x_{1,i} x_{1,j} \right]  
 + 2 \sum_{i \in [N]} \Ebb\left[ W_{2,1,i} x_{1,i} b_{2,1} \right] +  \Ebb\left[ b_{2,1}^2 \right] .
\end{align*}
The components of
$W_{2,:,:}$ and $x_{1,:}$, along with $b_{2,1}$,
are mutually independent,
so terms in the double sum with $i \neq j$ have zero expectation, 
and
$
\Ebb[h_{2,1}^2] = \left( \sum_{i \in [N]}  \Ebb\left[ W_{2,1,i}^2 \right] \Ebb \left[ x_{1,i}^2 \right] \right) + \Ebb[b_{2,1}^2].
$
For a random variable $x$ with the same distribution as the components
of $x_{1,:}$, this implies
\begin{equation}
\label{e:one.square}
\Ebb[h_{2,1}^2] = \sigma_w^2 \Ebb \left[ x^2 \right]+\sigma_b^2.
\end{equation}

Similarly,
\begin{align*}
& \Ebb[h_{2,1}^2 h_{2,2}^2] \\
    & = \Ebb\left[ \left[ \sum_{i \in [N]} W_{2,1,i} x_{1,i} + b_{2,1} \right]^2 \left[ \sum_{i \in [N]} W_{2,2,i} x_{1,i} + b_{2,2} \right]^2 \right]  \\
    &=\sum_{(i,j,r,s) \in [N]^4} \Ebb[ W_{2,1,i} W_{2,1,j} W_{2,2,r} W_{2,2,s} x_{1,i} x_{1,j} x_{1,r} x_{1,s} ] \\
    & \quad+ 2 \sum_{(i,j,r) \in [N]^3} \Ebb[ W_{2,1,i} W_{2,1,j} W_{2,2,r}  x_{1,i} x_{1,j} x_{1,r} b_{2,2} ] 
             \!+\! 2 \sum_{(i,r,s) \in [N]^3} \Ebb[ W_{2,1,i} W_{2,2,r} W_{2,2,s}  x_{1,i} x_{1,r} x_{1,s} b_{2,1} ] \\
    & \quad+ 4 \sum_{(i,r) \in [N]^2} \Ebb[ W_{2,1,i} W_{2,2,r}  x_{1,i} x_{1,r} b_{2,1} b_{2,2} ] \\
    & \quad+ \sum_{(i,j) \in [N]^2} \Ebb[ W_{2,1,i} W_{2,1,j}  x_{1,i} x_{1,j} b_{2,2}^2 ] 
           + \sum_{(r,s) \in [N]^2} \Ebb[ W_{2,2,r} W_{2,2,s}  x_{1,r} x_{1,s} b_{2,1}^2 ]  \\
    & \quad+ 2 \sum_{i \in [N]} \Ebb[ W_{2,1,i} x_{1,i} b_{2,1} b_{2,2}^2 ] 
           + 2 \sum_{r \in [N]} \Ebb[ W_{2,2,r} x_{1,r} b_{2,1}^2 b_{2,2} ]  \\
    & \quad+ \Ebb[ b_{2,1}^2 b_{2,2}^2 ]  \\
    &=\sum_{(i,r) \in [N]^2, i \neq r} \Ebb[ W_{2,1,i}^2 W_{2,2,r}^2] \Ebb[ x_{1,i}^2 ] \Ebb [x_{1,r}^2 ] 
        + \sum_{i \in [N]} \Ebb[ W_{2,1,i}^2 W_{2,2,i}^2] \Ebb[ x_{1,i}^4 ] \\
    & \quad+ \sum_{i \in [N]} \Ebb[ W_{2,1,i}^2]  \Ebb[ x_{1,i}^2 ] \Ebb[ b_{2,2}^2 ] 
           + \sum_{r \in [N]} \Ebb[ W_{2,2,r}^2 ] \Ebb[ x_{1,r}^2 ] \Ebb[ b_{2,1}^2 ]  \\
    & \quad+ \Ebb[ b_{1,2}^2 b_{2,2}^2 ]  \\
    & = \frac{(N^2 - N) \sigma_w^4 \Ebb[ x^2 ]^2}{N^2} + \frac{N \sigma_w^4  \Ebb[ x^4 ]}{N^2} +\frac{2N \sigma_w^2 \Ebb[x^2] \sigma_b^2}{N} + \sigma_b^4 \\
    & = \sigma_w^4 \Ebb[ x^2 ]^2 + \frac{\sigma_w^4  (\Ebb[ x^4 ] - \Ebb[ x^2 ]^2)}{N} + 2 \sigma_w^2 \sigma_b^2 \Ebb[x^2] + \sigma_b^4. \\
\end{align*} 

Putting this together with (\ref{e:one.square}), we have
\begin{equation}
\label{e:by.x.moments}
\Ebb[h_{2,1}^2 h_{2,2}^2] - \Ebb[h_{2,1}^2] \Ebb[h_{2,2}^2]
 = \frac{\sigma_w^4 ( \Ebb[ x^4 ] - \Ebb[ x^2 ]^2)}{N}.
\end{equation}

Now, we calculate the difference using \eqref{e:by.x.moments} for the Heaviside and ReLU functions.

{\bf Heaviside.}  Suppose $\phi$ is Heaviside function, i.e.\ $\phi(z)$ is the indicator function
for $z > 0$.  In this case, since the components of $h_{1,:}$ are symmetric about $0$,
the distribution of $x_{1,:}$ is uniform over $\{ 0,1 \}^N$.  Thus
$\Ebb[ x^4 ] = \Ebb[ x^2 ] = 1/2$, and so \eqref{e:by.x.moments} gives
$
\Ebb[h_{2,1}^2 h_{2,2}^2] - \Ebb[h_{2,1}^2] \Ebb[h_{2,2}^2] = \frac{3 \sigma_w^4 }{4 N} \neq 0.
$

{\bf ReLU.}  Next, we consider the case that $\phi$ is the ReLU.  
Recalling that, for all $i$, $h_{1,i} \sim \Gauss(0,\sigma_w^2)$, we have
$
\Ebb[ x^2 ] 
= \frac{1}{\sqrt{2\pi\sigma_w^2}}\int_0^\infty z^2 \exp\left(\frac{-z^2}{2\sigma_w^2}\right)dz.
$
By symmetry this is
$\frac{1}{2} \Ebb_{z \sim \Gauss(0,\sigma_w^2)}[z^2] = \sigma_w^2/2$.  
Similarly, $\Ebb[ x^4 ] = 
             \frac{1}{2} \Ebb_{z \sim \Gauss(0,\sigma_w^2)}[z^4] = \frac{3 \sigma^4}{2}$.
Plugging these into \eqref{e:by.x.moments} we get that, in the case the $\phi$
is the ReLU, that
\[
\Ebb[h_{2,1}^2 h_{2,2}^2] - \Ebb[h_{2,1}^2] \Ebb[h_{2,2}^2]
 = \frac{\sigma_w^4 \left((3/2) \sigma_w^4 -  \sigma_w^4/4\right)}{N}
 = \frac{5 \sigma_w^8}{4 N}
 > 0,
\]
completing the proof.
\end{proof}

Note that, informally, the degree of dependence established in
the proof of Theorem~\ref{thm:independence} approaches $0$
as $N$ gets large.  

\subsection{Undefined length map}
Here, we show, informally, that for
$\phi$ at the boundary of the second
condition in the definition of permissibility,
the recursive formula defining the length
map $\tq_{\ell}$ breaks down.
Roughly, this condition cannot be relaxed.

\begin{proposition}
\label{p:undefined}
For any $\alpha > 0$, if $\phi$ is defined by
$\phi(x)=\exp(\alpha x^2)$, there exists a $\sigma_w, \sigma_b$ s.t.\ $\tq_{\ell}, \tr_{\ell}$ is undefined for all $\ell \geq 2$.
\end{proposition} 

\begin{proof}
Suppose $\sigma^2_w + \sigma_b^2 = \frac{1}{4 \alpha^2}$.  
Then $\tq_1 = \frac{1}{4 \alpha^2}$, so
that
\begin{align*}
\tr_1 
  & =\frac{1}{\sqrt{2\pi}} \int_{-\infty}^{\infty} \phi(\sqrt{\tq_1} z) \exp\left(-\frac{z^2}{2}\right)dz  
 =\frac{1}{\sqrt{2\pi}} \int_{-\infty}^{\infty} \exp( \alpha \sqrt{\tq_1} z^2) \exp\left(-\frac{z^2}{2}\right)dz  
  \\
  & 
 =\frac{1}{\sqrt{2\pi}} \int_{-\infty}^{\infty} \exp( z^2/2) \exp\left(-\frac{z^2}{2}\right)dz  
 = \infty,
\end{align*}
and downsteam values of $\tq_{\ell}$ and $\tr_{\ell}$ are undefined.
\end{proof}

%

\section{Experiments}
\label{s:experiments}

Our first experiment fixed $x[0,:] = (1,...,1)$,
$\sigma_w = 1$, $\sigma_b = 0$, $\phi(z) = 1/z$.

For each $N \in \{ 10, 100, 1000 \}$, we 
(a) initialized the
weights $100$ times, (b) plotted the histograms
of all of the values of $h[2,:]$, along with the
$\Cauchy(0,\sqrt{N})$ distribution from
the proof of Proposition~\ref{p:not.gaussian},
and $\Gauss(0,\sigma^2)$ for $\sigma$
estimated from the data.
\begin{figure}[!ht]
    \centering
    \begin{subfigure}[b]{0.3\textwidth}
        \includegraphics[width=\textwidth]{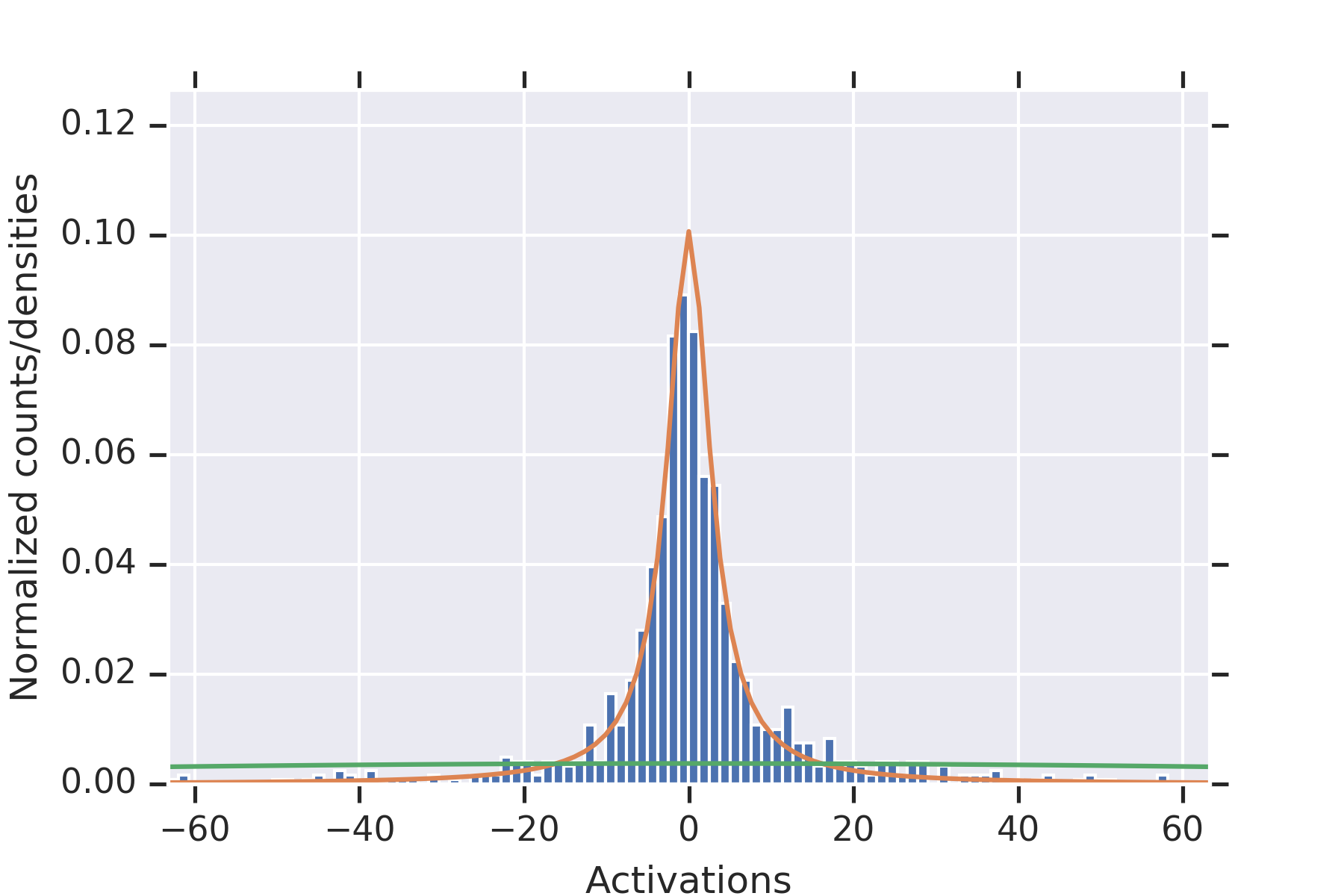}
        \caption{$N=10$}
        \label{fig:10}
    \end{subfigure}
    ~ 
    \begin{subfigure}[b]{0.3\textwidth}
        \includegraphics[width=\textwidth]{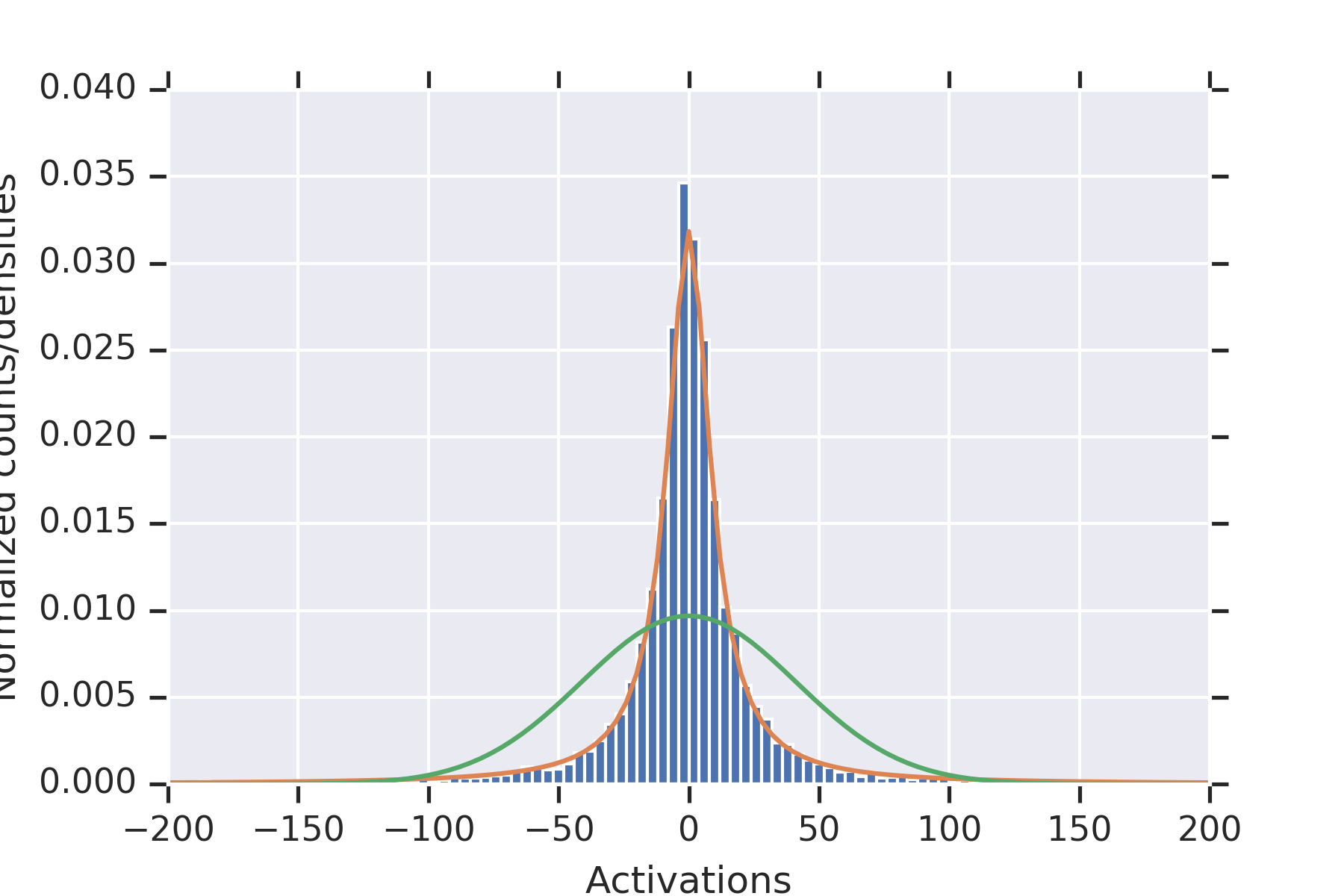}
        \caption{$N=100$}
        \label{fig:100}
    \end{subfigure}
    ~ 
    \begin{subfigure}[b]{0.3\textwidth}
        \includegraphics[width=\textwidth]{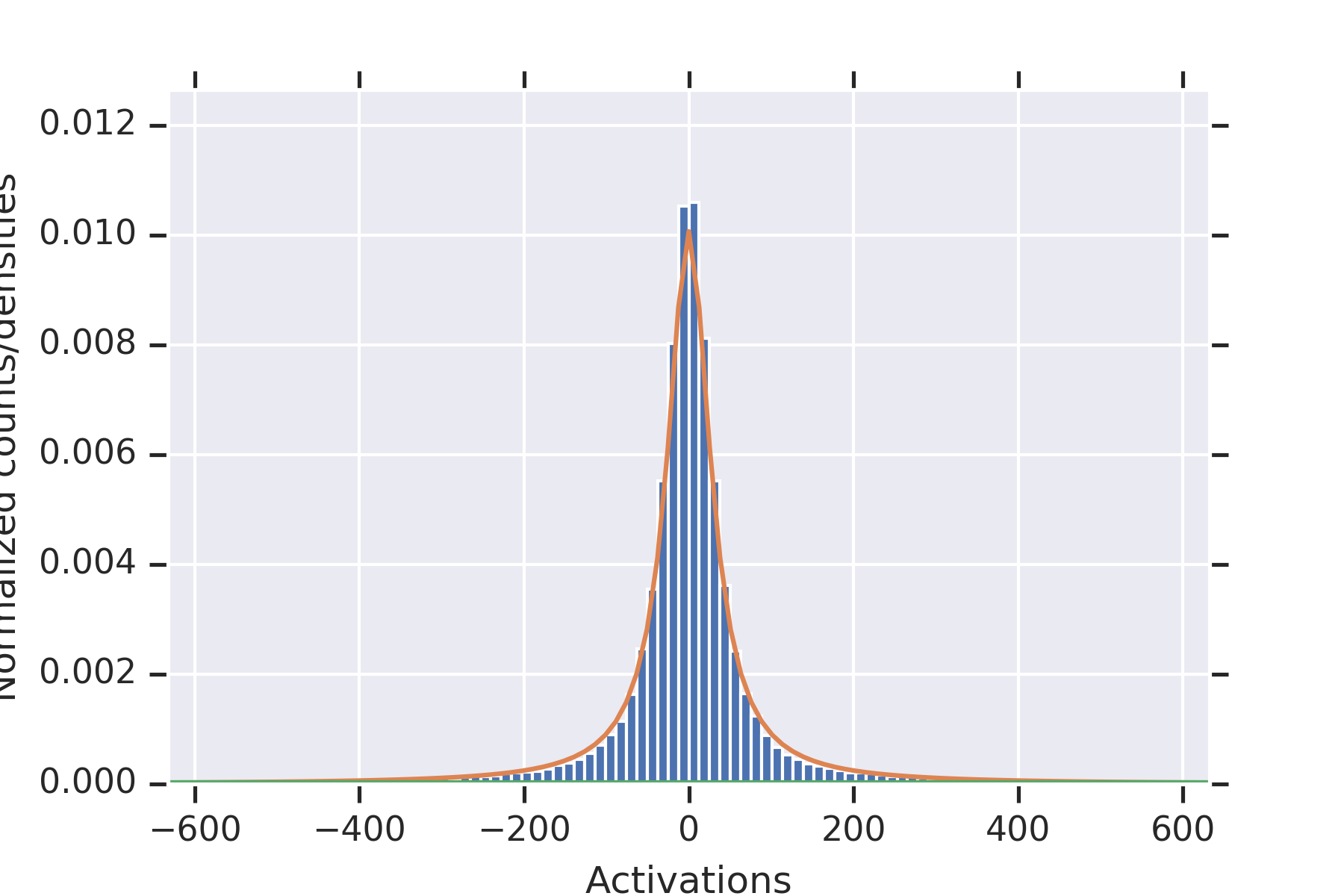}
        \caption{$N=1000$}
        \label{fig:1000}
    \end{subfigure}
    \caption{Histograms of $h[2,:]$,
             averaged over 100 random initializations,
             for $N \in \{ 10, 100, 1000 \}$, along
             with $\Cauchy(0,\sqrt{N})$ (shown in green)
             and  $\Gauss(0,\sigma^2)$ for $\sigma$ estimated
             from the data (shown in red).  When we average over
             multiple random initializations of the weights,
             the distribution of the activations matches
             the Cauchy distribution, and not the Gaussian.
             }
\label{fig:merge_sample}
\end{figure}
Consistent with the theory, the $\Cauchy(0,\sqrt{N})$ distribution
fits the data well.

To illustrate the fact that the values in the second hidden layer are
not independent, for $N = 1000$ and the parameters otherwise as in
the other experiment, we plotted histograms of the values seen in
the second layer for nine random initializations of the weights
in Figure~\ref{fig:separate_sample}.  
When some of the values in the first hidden layer have
unusually small magnitude, then
the values in the second hidden layer coordinately tend to be large.
\begin{figure}[!ht]
\begin{center}
    \begin{subfigure}[b]{0.3\textwidth}
        \includegraphics[width=\textwidth]{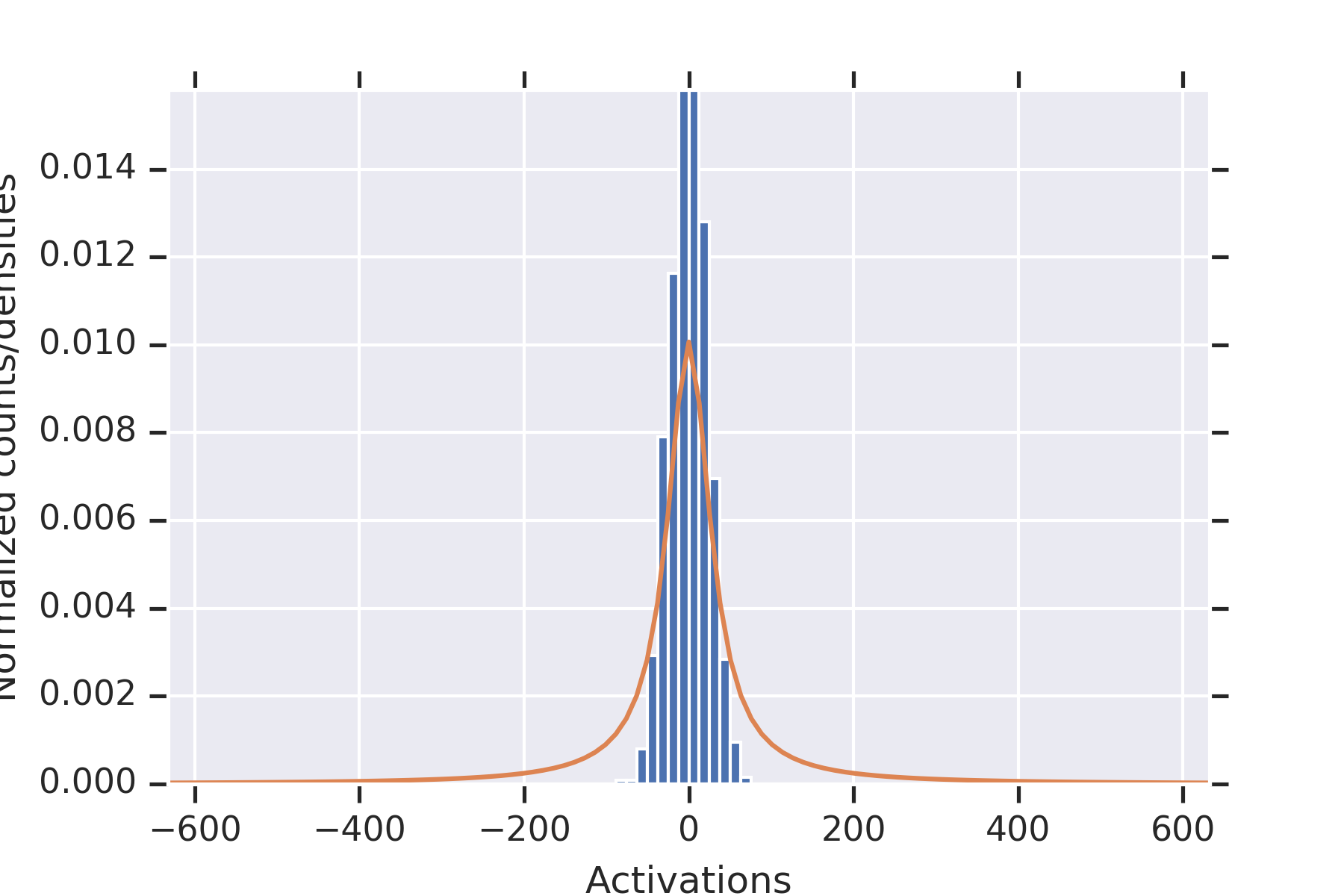}
    \end{subfigure}
    ~ 
    \begin{subfigure}[b]{0.3\textwidth}
        \includegraphics[width=\textwidth]{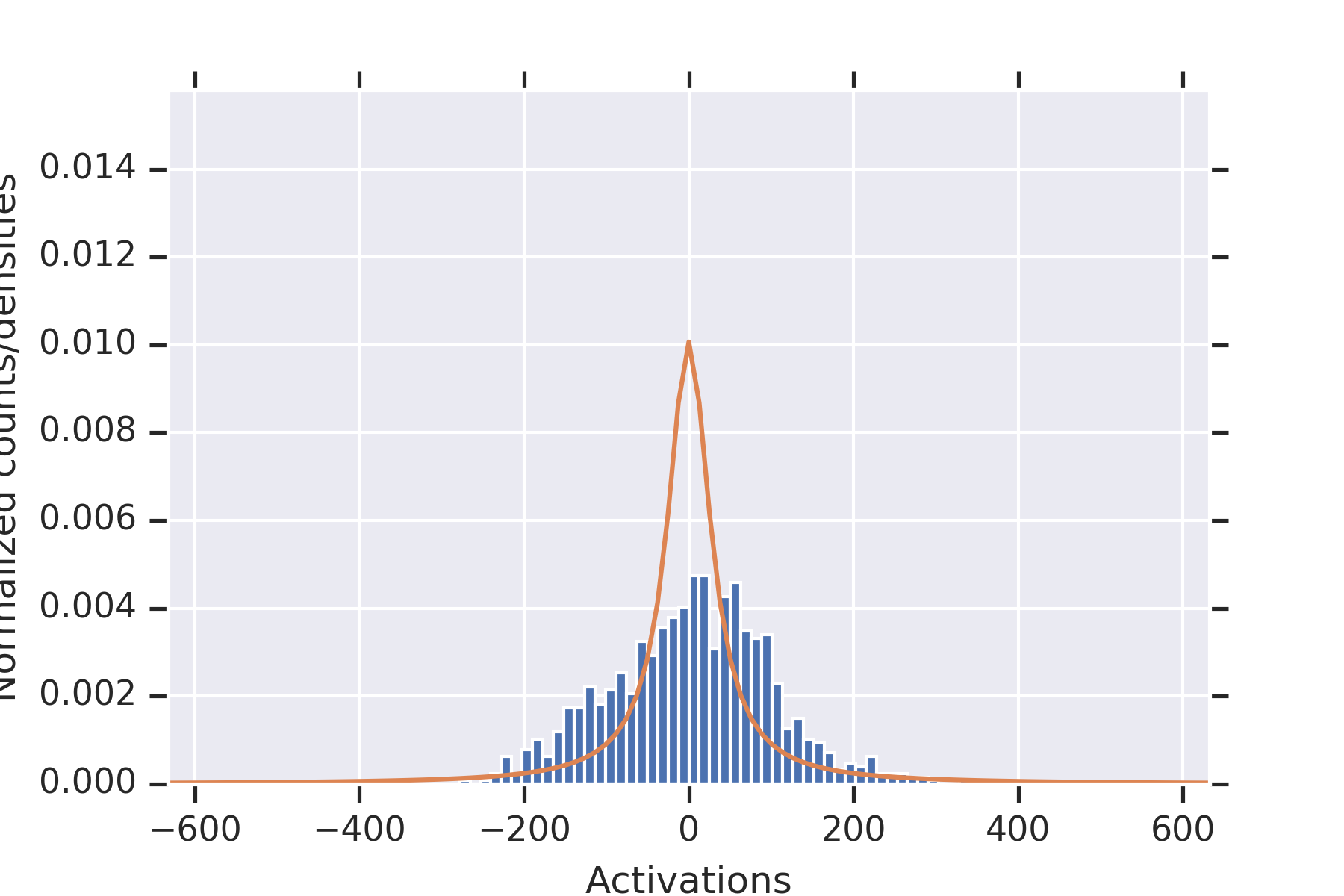}
    \end{subfigure}
    ~ 
    \begin{subfigure}[b]{0.3\textwidth}
        \includegraphics[width=\textwidth]{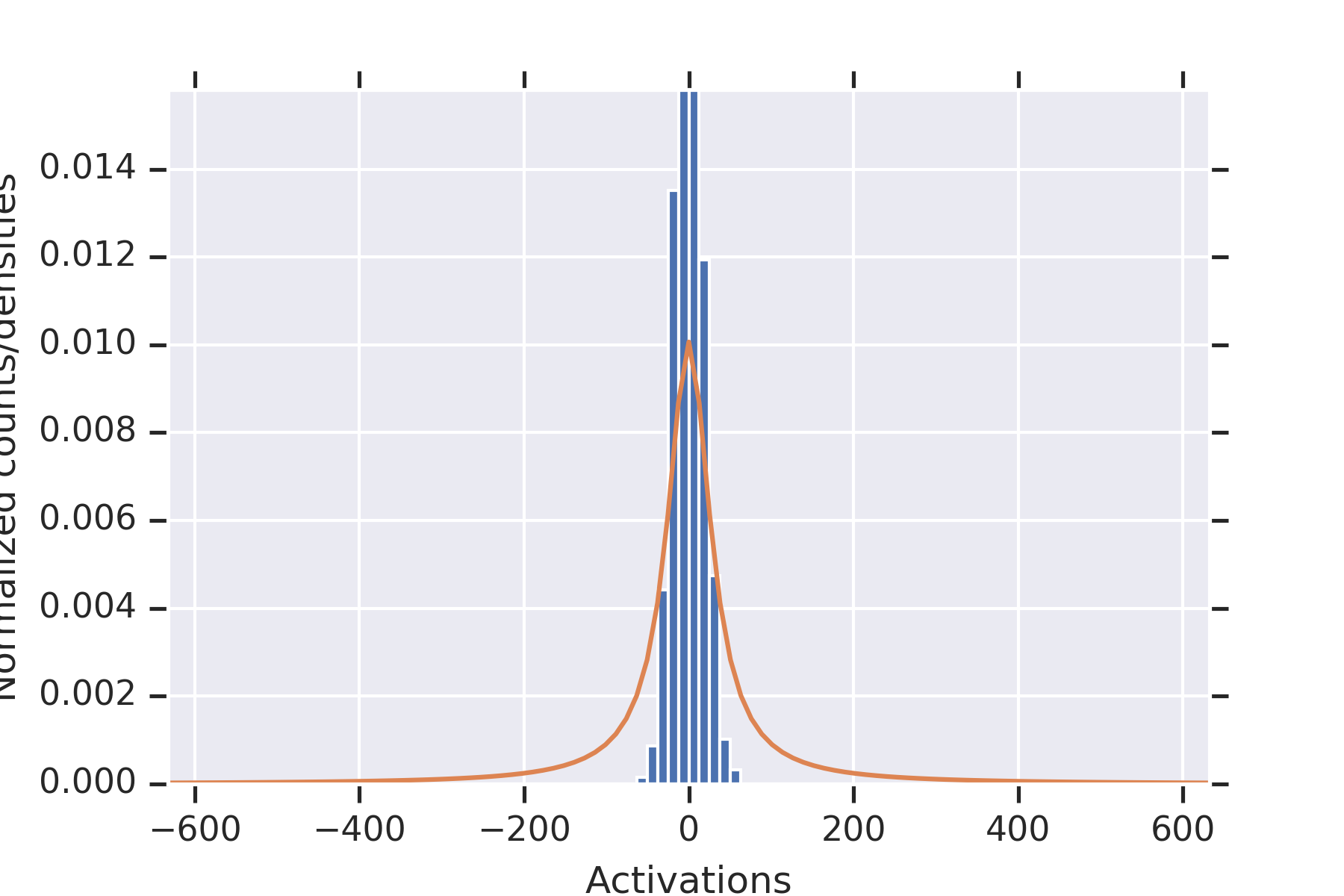}
    \end{subfigure} \\
    \begin{subfigure}[b]{0.3\textwidth}
        \includegraphics[width=\textwidth]{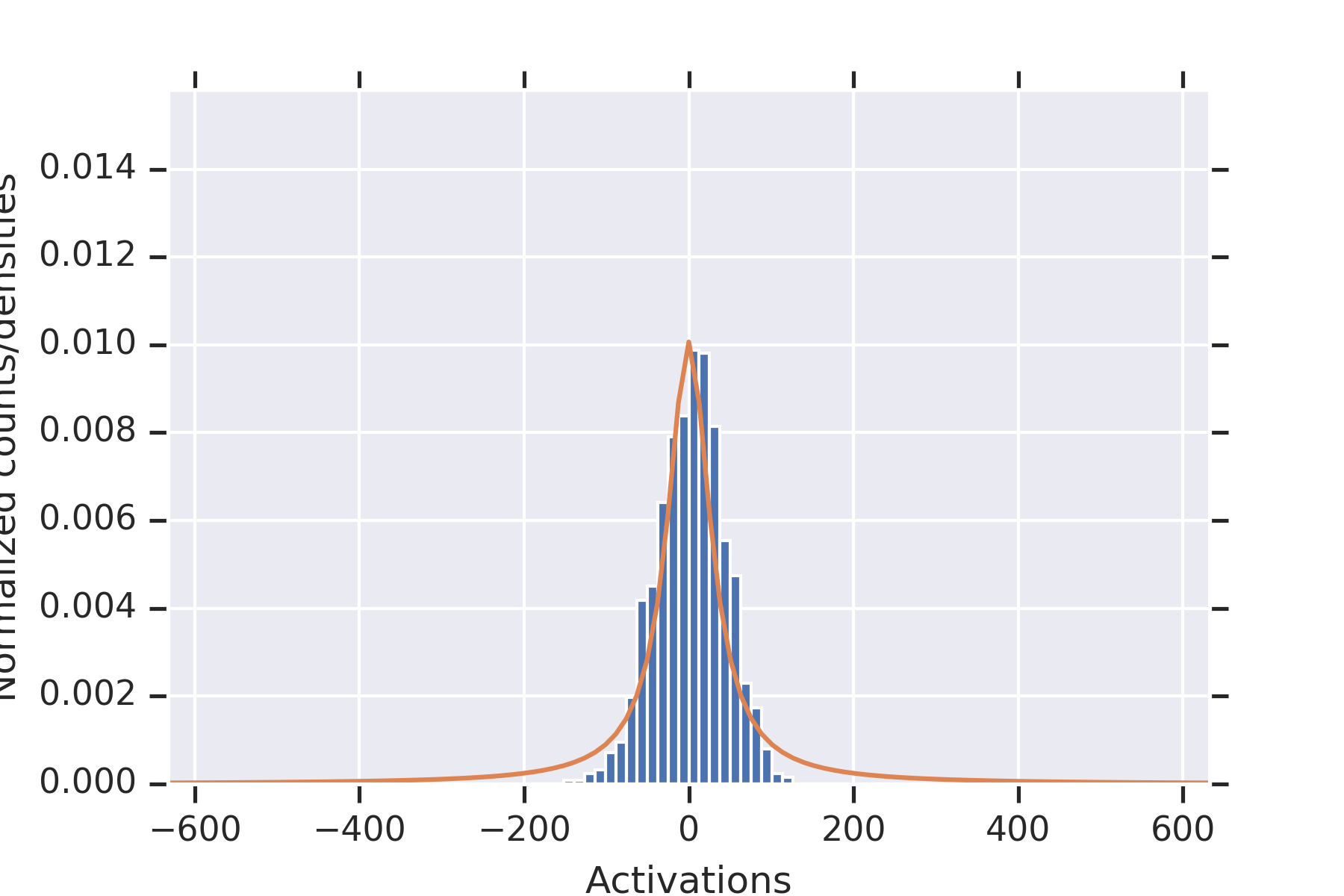}
    \end{subfigure}
    ~ 
    \begin{subfigure}[b]{0.3\textwidth}
        \includegraphics[width=\textwidth]{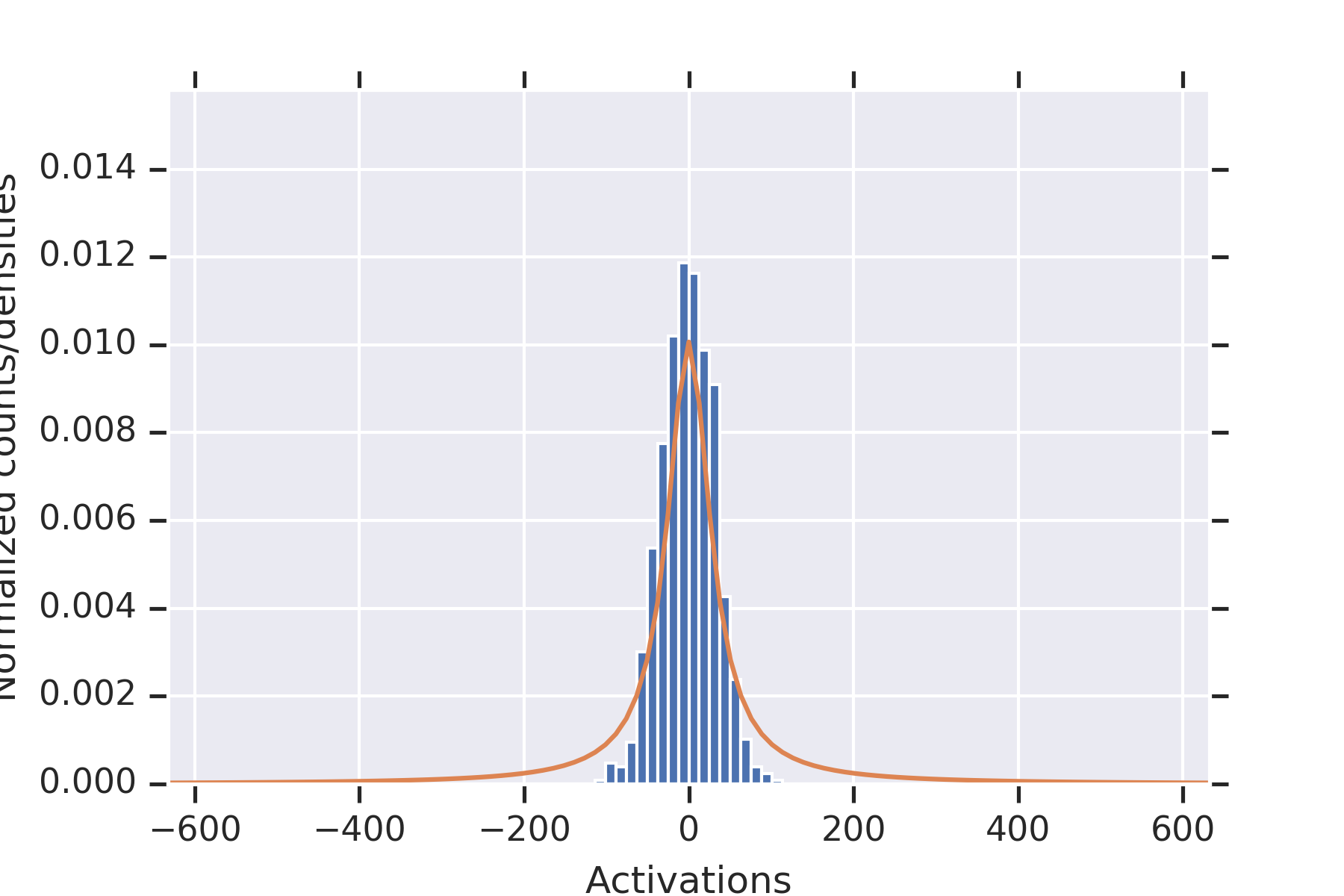}
    \end{subfigure}
    ~ 
    \begin{subfigure}[b]{0.3\textwidth}
        \includegraphics[width=\textwidth]{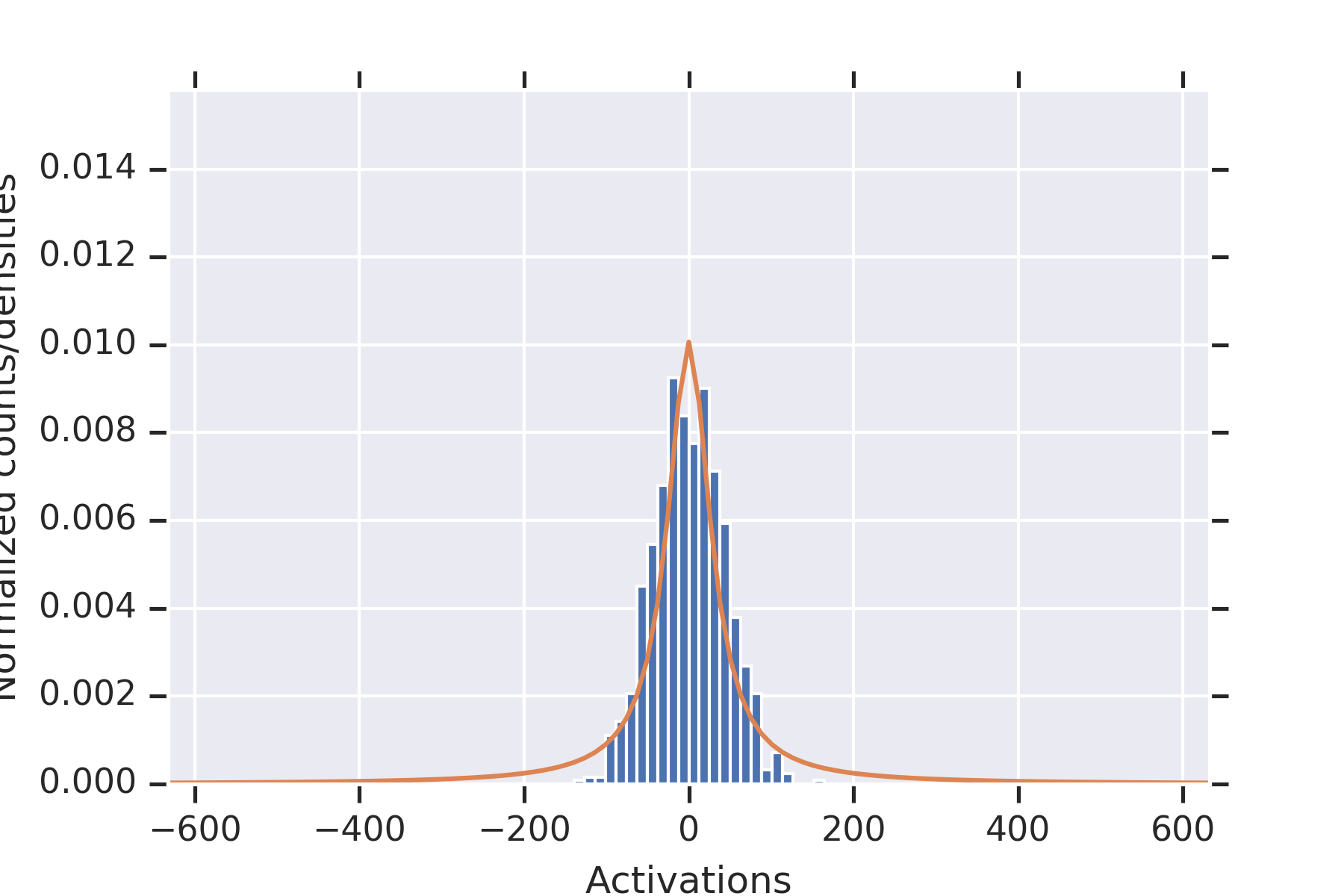}
    \end{subfigure} \\
    \begin{subfigure}[b]{0.3\textwidth}
        \includegraphics[width=\textwidth]{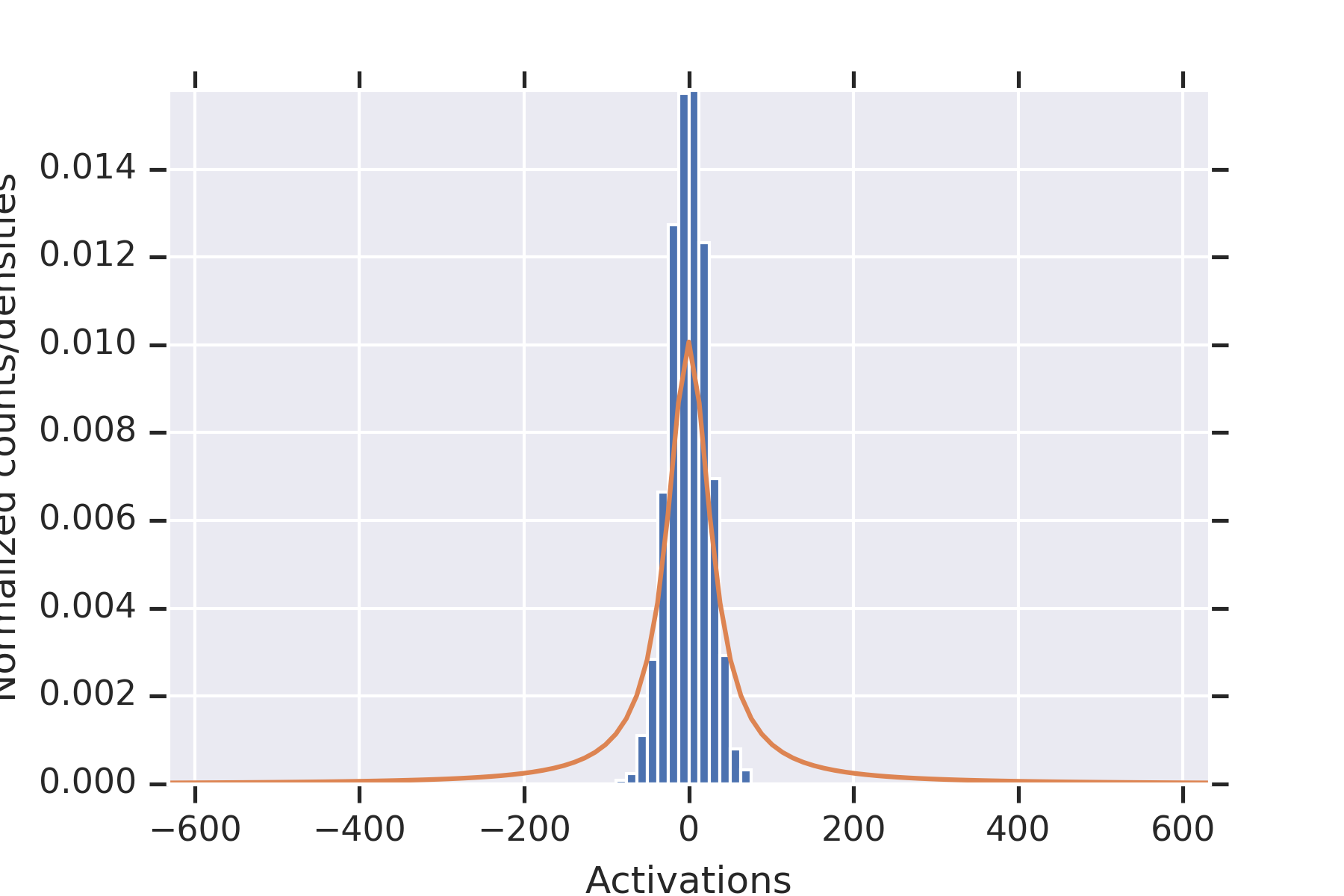}
    \end{subfigure}
    ~ 
    \begin{subfigure}[b]{0.3\textwidth}
        \includegraphics[width=\textwidth]{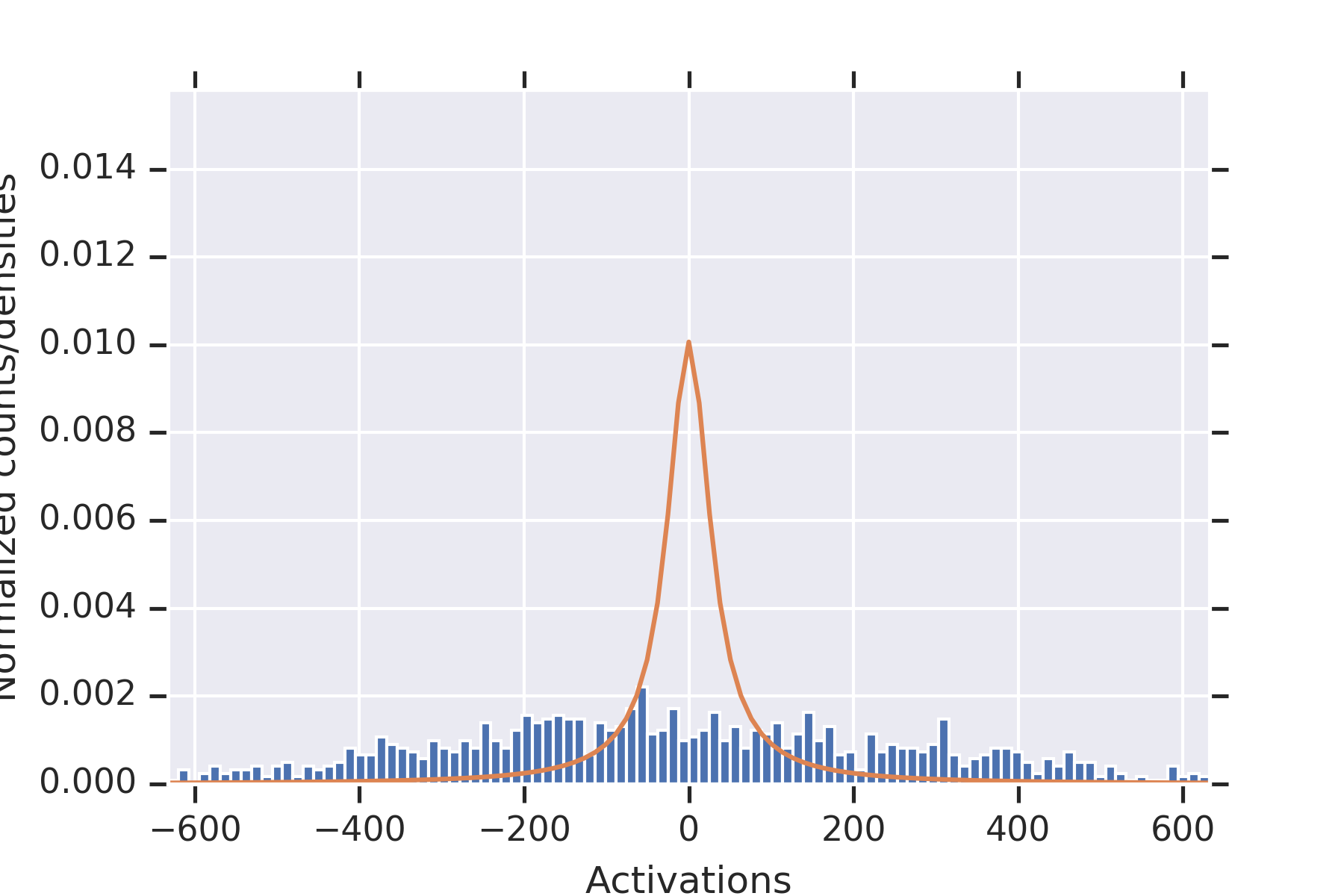}
    \end{subfigure}
    ~ 
    \begin{subfigure}[b]{0.3\textwidth}
        \includegraphics[width=\textwidth]{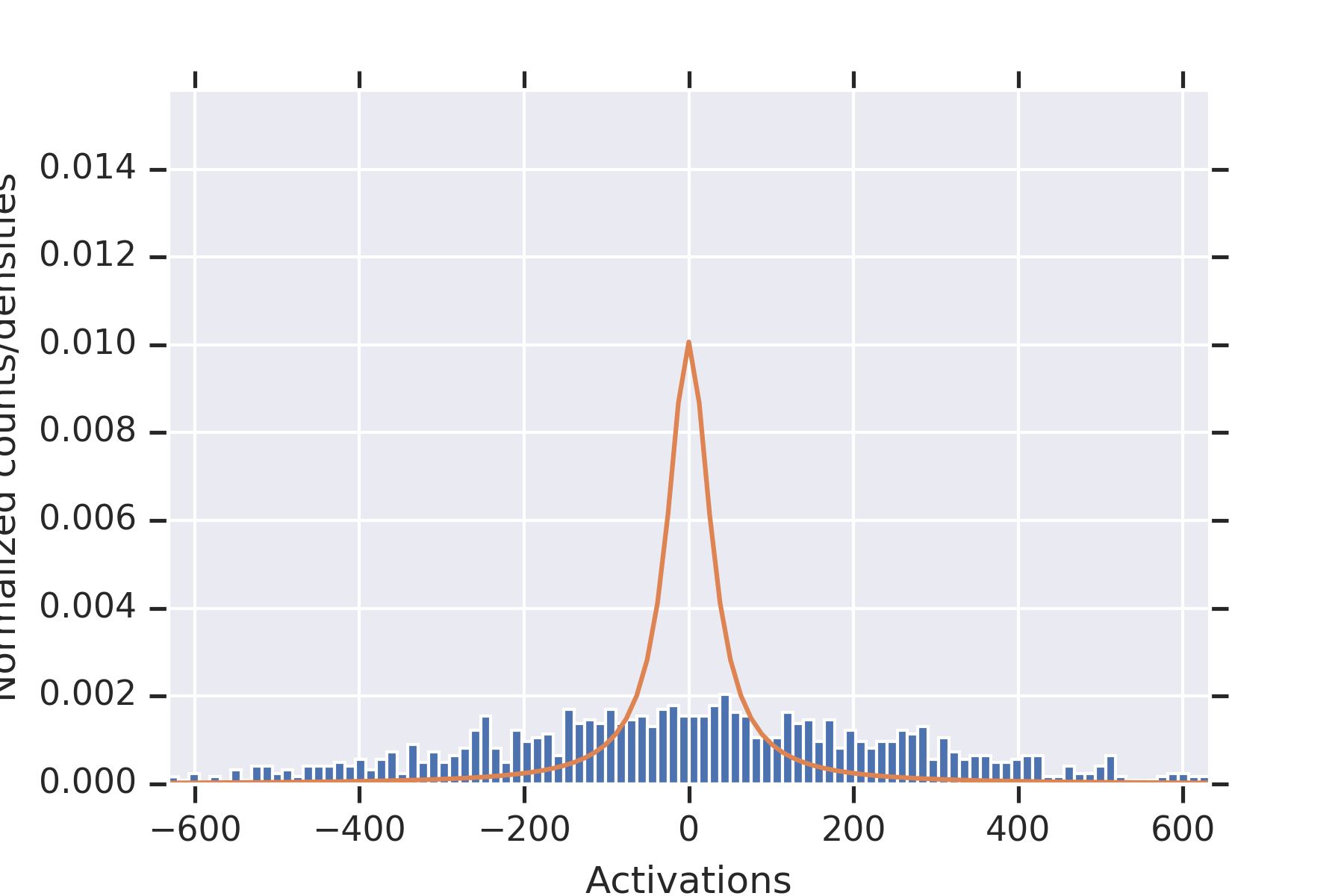}
    \end{subfigure} \\
\end{center}
    \caption{Histograms of $h[2,:]$ for nine random
             weight initializations.  Plotting activations
             separately for different random initializations
             reveals the dependence among the activations in
             a layer.}
\label{fig:separate_sample}
\end{figure}
Note that this is consistent with Theorem~\ref{t:conv.prob}
establishing convergence in probability for permissible
$\phi$, since the $\phi$ used in this experiment is not
permissible.

\appendix

\section{Proof of Lemma~\protect\ref{l:integrable}}
\label{a:integrable}

Choose $c > 0$.
Since $\mathrm{limsup}_{x \rightarrow \infty} \frac{\log | \phi(x) |}{x^2} = 0$
and $\mathrm{limsup}_{x \rightarrow -\infty} \frac{\log | \phi(x) |}{x^2} = 0$, 
we also have \\
$\mathrm{limsup}_{x \rightarrow \infty} \frac{\log | \phi(c x) |}{x^2} = 0$
and $\mathrm{limsup}_{x \rightarrow -\infty} \frac{\log | \phi(c x) |}{x^2} = 0$.
Thus, there
is an $a$ such that, for all 
$x \not\in [-a,a]$, $\log | \phi(c x) | \leq \frac{x^2}{8}$, which
implies $\phi(c x)^2 \leq \exp\left(\frac{x^2}{4}\right)$.  
Since $\phi$ is permissible, it is bounded on $[-a,a]$. Thus, we have
\begin{align*}
& \int \phi(c x)^2 \exp(-x^2/2) \;dx \\
&  = \int_{-\infty}^{-a} \phi(c x)^2 \exp(-x^2/2) dx + 
\int_{-a}^{a} \phi(c x)^2 \exp(-x^2/2) dx + 
 \int_{a}^{\infty} \phi(c x)^2 \exp(-x^2/2) dx \\
& \leq \int_{-\infty}^{-a} \exp(-x^2/4) dx + 
  \left( \sup_{x \in [-a, a]} \phi(c x)^2 \right)\int_{-a}^{a} \exp(-x^2/2) dx
 + \int_{a}^{\infty} \exp(-x^2/4) dx \\
& < \infty
\end{align*}
completing the proof.

\section{Proof of Lemma~\protect\ref{l:non_zero}}
\label{a:non_zero}

The proof is by induction.  The base case holds since
$\tq_{0} = \tr_0 = 1$.  

To prove the inductive step, we need the following lemma.
\begin{lemma}
\label{l:nonzero.implies.nonzero}
If $\phi$ is not zero a.e., then, for all $c > 0$,
$\Ebb_{z \in \Gauss(0,1)} (\phi(c z)^2) > 0$.
\end{lemma}
\begin{proof}
If $\mu$ is the Lebesgue measure, since
\[
\mu( \{ x \in \R : \phi^2(c x) > 0 \} )
 =
\lim_{n \rightarrow \infty}
   \mu( \{ x : \phi^2(c x) > 1/n \} \cap [-n,n] ) > 0,
\]
there exists $n$ such that 
$\mu( \{ x : \phi^2(c x) > 1/n \} \cap [-n,n] ) > 0$.  
For such an $n$, we have
\[
\Ebb_{z \in \Gauss(0,1)} (\phi(c z)^2)
 \geq 
  \frac{1}{n}  e^{-n^2/2} \mu( \{ x : \phi^2(c x) > 1/n \} \cap [-n,n] ) > 0.
\]
\end{proof}

\medskip

Returning to the proof of Lemma~\ref{l:non_zero}, by the inductive
hypothesis,
$\tr_{\ell-1} > 0$, which, since $\sigma_w > 0$, implies
$\tq_{\ell} > 0$. 
Applying Lemma~\ref{l:nonzero.implies.nonzero}
  yields $\tr_{\ell} > 0$.

\section{Proof of Lemma~\protect\ref{l:tail}}
\label{a:tail}

Since $\mathrm{limsup}_{x \rightarrow \infty} \frac{\log | \phi(x) |}{x^2} = 0$
there is an $b$ such that, for all
$x \geq b$, $\log | \phi(x) | \leq \frac{x^2}{8 s}$, which
implies
$\phi(x)^2 \leq \exp\left(\frac{x^2}{ 4 s}\right)$.  
Now, choose $q \in [r,s]$.  For $a = b /\sqrt{r}$, we then have
\begin{align*}
& \int_a^{\infty} \phi(\sqrt{q} x)^2 \exp(-x^2/2)\;dx \\
& = \frac{1}{\sqrt{q}} \int_{a \sqrt{q}}^{\infty} 
      \phi(z)^2 \exp\left(-\frac{z^2}{2 q} \right) \;dz \\
& \leq \frac{1}{\sqrt{q}} \int_{a \sqrt{q}}^{\infty} 
  \exp\left(\frac{z^2}{ 4 s}\right)
       \exp\left(-\frac{z^2}{2 q} \right) \;dz \\
& \leq \frac{1}{\sqrt{q}} \int_{a \sqrt{q}}^{\infty} 
       \exp\left(-\frac{z^2}{4 q} \right) \;dz \\
& \leq \frac{1}{\sqrt{q}} \int_{b}^{\infty} 
       \exp\left(-\frac{z^2}{4 q} \right) \;dz. \\
\end{align*}
By increasing $b$ if necessary, we can ensure
$\frac{1}{\sqrt{q}} \int_{b}^{\infty} 
       \exp\left(-\frac{z^2}{4 q} \right) \;dz
  \leq \beta$
which then gives \\
$\int_a^{\infty} \phi(\sqrt{q} x)^2 \exp(-x^2/2)\;dx
 \leq \beta$.  A symmetric argument yields
$\int_{-\infty}^a \phi(\sqrt{q} x)^2 \exp(-x^2/2)\;dx
 \leq \beta$, completing the proof.

\end{document}

%% file: macros.tex
\usepackage{amsmath,amsbsy,amsfonts,amssymb,dsfont,units,psfrag}
\usepackage{amssymb}
\usepackage{mathtools}

\usepackage{algorithm,algorithmic,mathtools}
\usepackage{enumitem}
\usepackage{color,cases,caption}

\usepackage{subcaption}
\usepackage{float}
\usepackage{tikz,lipsum}
\usetikzlibrary{calc, shapes, backgrounds, arrows, fit, calc, positioning}

 \newcommand{\IGNORE}[1]{}

\newcommand{\R}{\mathbb{R}}
\newcommand{\N}{\mathbb{N}}

\newcommand*{\eqdef}{\stackrel{\textup{def}}{=}}

\newcommand{\tr}{\mathrm{tr}}

 \def\0{{\bf 0}}

\def\qed{\hfill\hbox{${\vcenter{\vbox{
    \hrule height 0.4pt\hbox{\vrule width 0.4pt height 6pt
    \kern5pt\vrule width 0.4pt}\hrule height 0.4pt}}}$}}

\definecolor{dkr}{rgb}{0.6,0.2,0.2}
\definecolor{dkg}{HTML}{00CC00}
\definecolor{dkb}{rgb}{0.0,0.3,0.7}
\definecolor{blue1}{HTML}{0066FF}
\definecolor{brm}{rgb}{1,0.0,1}
\definecolor{lpurple}{cmyk}{.05,0.18,0,0}
\definecolor{dred}{cmyk}{0,1,.81,.04}
\definecolor{lred}{cmyk}{0,0.26,.1,0}
\definecolor{dgreen}{cmyk}{1,0,.53,.21}
\definecolor{lgreen}{cmyk}{.38,0,.36,0}
\definecolor{dblue}{cmyk}{1,0.44,0,0}
\definecolor{lblue}{cmyk}{.25,0.02,0,0}

\def\Ebb{{\mathbb E}}

\def\beq{\begin{equation}}
\def\eeq{\end{equation}\noindent}
\def\beqn{\begin{eqnarray}}
\def\eeqn{\end{eqnarray} \noindent}
\def\beqnn{  \begin{eqnarray*}}
\def\eeqnn{\end{eqnarray*}  \noindent}
\def\bcase{  \begin{numcases}}
\def\ecase{\end{numcases}   \noindent}
\def\bsbcase{  \begin{subnumcases}}
\def\esbcase{\end{subnumcases}   \noindent}

\newtheorem{theorem}{Theorem}

\newtheorem{lemma}[theorem]{Lemma}

\newtheorem{proposition}[theorem]{Proposition}

\newtheorem{definition}{Definition}
\newenvironment{proof}{\noindent{\bf Proof}:}{\qed}

\newcommand{\matplottc}[1]{               
        \unitlength .45truein
        \begin{center}
        \includegraphics{#1.ps}
        \end{picture}
        \end{center}
}

\def\psfancypar#1#2{\begingroup\def\par{\endgraf\endgroup\lineskiplimit=0pt}
               \setbox2=\hbox{\large\sc #2}
               \newdimen\tmpht \tmpht \ht2 \advance\tmpht by \baselineskip
               \font\hhuge=Times-Bold at \tmpht
               \setbox1=\hbox{{\hhuge #1}}
               \count7=\tmpht \count8=\ht1
               \divide\count8 by 1000 \divide\count7 by \count8
               \tmpht=.001\tmpht\multiply\tmpht by \count7
               \font\hhuge=Times-Bold at \tmpht
               \setbox1=\hbox{{\hhuge #1}}
               \noindent
                \hangindent1.05\wd1
               \hangafter=-2 {\hskip-\hangindent
               \lower1\ht1\hbox{\raise1.0\ht2\copy1}%
                \kern-0\wd1}\copy2\lineskiplimit=-1000pt}

\def\Kout{\setbox1=\hbox{\Huge\bf K}\hbox to
1.05\wd1{\hspace{.05\wd1}
}}
\def\Sout{\setbox1=\hbox{\Huge\bf S}\hbox to 1.05\wd1{\hspace{.05\wd1}
}}

\newcommand{\torestate}[3]{%
\expandafter \def \csname BBRESTATE #2 \endcsname{#3}
\theoremstyle{plain}
\newtheorem{BBRESTATETHMNUM#2}[theorem]{#1}
\begin{BBRESTATETHMNUM#2}\label{#2}\csname BBRESTATE #2 \endcsname   \end{BBRESTATETHMNUM#2}
\newtheorem*{BBRESTATETHMNONNUM#2}{{#1}~\ref{#2}}
}

\newcommand{\restate}[1]{\begin{BBRESTATETHMNONNUM#1}[Restated] \csname BBRESTATE #1 \endcsname
\end{BBRESTATETHMNONNUM#1}}

\definecolor{blue1}{HTML}{0066FF}
\definecolor{lpurple}{cmyk}{.05,0.18,0,0}

\newcommand{\ux}{\underline{x}}
\newcommand{\uq}{\underline{q}}
\newcommand{\oq}{\overline{q}}
\newcommand{\uh}{\underline{h}}
\newcommand{\tq}{\tilde{q}}
\renewcommand{\tr}{\tilde{r}}